\documentclass[10pt]{article}

\usepackage{tikz}

\usepackage{xspace}

\newcommand{\term}[1]{\text{\tt{#1}}\xspace}

\usepackage{amsmath,amssymb}
\usepackage{amsthm}
\usepackage{mathtools}
\usepackage[noend]{algorithmic}
\usepackage[ruled,vlined]{algorithm2e}
\usepackage{url}
\usepackage{fullpage}
\usepackage{makeidx}
\usepackage{enumerate}

\usepackage{hyperref}
\usepackage{epstopdf}
\usepackage{color}
\usepackage{xr}

\usepackage{caption}
\usepackage[utf8]{inputenc}
\usepackage{thm-restate}

\usepackage{scalerel,stackengine}

\usepackage{enumitem}

\stackMath
\newcommand\reallywidehat[1]{%
\savestack{\tmpbox}{\stretchto{%
  \scaleto{%
    \scalerel*[\widthof{\ensuremath{#1}}]{\kern.1pt\mathchar"0362\kern.1pt}%
    {\rule{0ex}{\textheight}}
  }{\textheight}%
}{2.4ex}}%
\stackon[-6.9pt]{#1}{\tmpbox}%
}
\usepackage[mathscr]{euscript}

\numberwithin{equation}{section}

\newtheoremstyle{myexample} 
    {\topsep}                    
    {\topsep}                    
    {\rm }                   
    {}                           
    {\bf }                   
    {.}                          
    {.5em}                       
    {}  

\newtheoremstyle{myremark} 
    {\topsep}                    
    {\topsep}                    
    {\rm}                        
    {}                           
    {\bf}                        
    {.}                          
    {.5em}                       
    {}  

\newtheorem{claim}{Claim}[section]
\newtheorem{lemma}[claim]{Lemma}

\newtheorem{theorem}{Theorem}
\newtheorem{proposition}[claim]{Proposition}
\newtheorem{corollary}[claim]{Corollary}
\newtheorem{definition}[claim]{Definition}

\theoremstyle{myremark}
\newtheorem{remark}{Remark}[section]

\theoremstyle{myremark}

\theoremstyle{myexample}

\usepackage[suppress]{color-edits}
\addauthor{ms}{blue}

\newcommand{\E}{\mathbb E}
\newcommand{\Prob}{\mathbb P}
\newcommand{\R}{\mathbb R}

\newcommand{\slab}{\term{slab}}

\newcommand{\Rad}{\text{Rad}}
\def\bw{{\boldsymbol{w}}}

\begin{document}

\title{A Universal Law of Robustness via Isoperimetry}

\author{S\'ebastien Bubeck \\
Microsoft Research
\and Mark Sellke\\
Stanford University }


\date{}

\maketitle

\abstract{
Classically, data interpolation with a parametrized model class is possible as long as the number of parameters is larger than the number of equations to be satisfied. A puzzling phenomenon in deep learning is that models are trained with many more parameters than what this classical theory would suggest. We propose a partial theoretical explanation for this phenomenon. We prove that for a broad class of data distributions and model classes, overparametrization is {\em necessary} if one wants to interpolate the data {\em smoothly}. Namely we show that {\em smooth} interpolation requires $d$ times more parameters than mere interpolation, where $d$ is the ambient data dimension. We prove this universal law of robustness for any smoothly parametrized function class with polynomial size weights, and any covariate distribution verifying isoperimetry (or a mixture thereof). In the case of two-layer neural networks and Gaussian covariates, this law was conjectured in prior work by Bubeck, Li and Nagaraj. We also give an interpretation of our result as an improved generalization bound for model classes consisting of smooth functions.
}

\section{Introduction}
Solving $n$ equations generically requires only $n$ unknowns\footnote{As in, for instance, the inverse function theorem in analysis or B\'ezout's theorem in algebraic geometry. See also \cite{yun2019small, BELM20} for versions of this claim with neural networks.}. However, the revolutionary deep learning methodology revolves around highly overparametrized models, with many more than $n$ parameters to learn from $n$ training data points. We propose an explanation for this enigmatic phenomenon, showing in great generality that finding a \emph{smooth} function to fit $d$-dimensional data requires at least $nd$ parameters. In other words, overparametrization by a factor of $d$ is {\em necessary} for {\em smooth} interpolation, suggesting that perhaps the large size of the models used in deep learning is a {\em necessity} rather than a weakness of the framework. Another way to phrase the result is as a {\em tradeoff} between the size of a model (as measured by the number of parameters) and its ``robustness'' (as measured by its Lipschitz constant): either one has a small model (with $n$ parameters) which must then be non-robust, or one has a robust model (constant Lipschitz) but then it must be very large (with $nd$ parameters). Such a tradeoff was conjectured for the specific case of two-layer neural networks and Gaussian data in \cite{bubeck2020law}. Our result shows that in fact it is a \msedit{much more general} phenomenon which applies to essentially any parametrized function class (including in particular deep neural networks) as well as a much broader class of data distributions. As conjectured in \cite{bubeck2020law} we obtain an entire tradeoff curve between size and robustness: our universal law of robustness states that, for any function class smoothly parametrized by $p$ parameters, and for any $d$-dimensional dataset satisfying a \msedit{natural isoperimetry} condition, any function in this class that fits the data {\em below the noise level} must have (Euclidean) Lipschitz constant of order at least $\sqrt{\frac{nd}{p}}$.

\begin{theorem}[Informal version of Theorem~\ref{thm:main}]\label{thm:inf}

Let $\mathcal F$ be a class of functions from $\R^d\to \R$ and let $(x_i,y_i)_{i\in [n]}$ be i.i.d. input-output pairs in $\R^d\times [-1,1]$. Assume that:
\begin{enumerate}
    \item $\mathcal F$ admits a Lipschitz parametrization by $p$ real parameters, each of size at most $\mathrm{poly}(n,d)$.
    \item The distribution $\mu$ of the covariates $x_i$ satisfies isoperimetry (or is a mixture theoreof).
    \item The expected conditional variance of the output (i.e., the ``noise level'') is strictly positive, denoted $\sigma^2 \equiv \E^{\mu}[Var[y|x]] > 0$.
\end{enumerate}    
Then, with high probability over the sampling of the data, one has simultaneously for all $f \in \mathcal F$:
\[
  \frac{1}{n}\sum_{i=1}^n (f(x_i)-y_i)^2 \leq \sigma^2-\epsilon \,\, \Rightarrow \,\, \mathrm{Lip}(f) \geq \widetilde{\Omega}\left(\frac{\epsilon}{\sigma}\sqrt{\frac{nd}{p}}\right)  \,.
\]
\end{theorem}

\begin{remark}\label{rem:exists} For the distributions $\mu$ we have in mind, for instance uniform on the unit sphere, there exists with high probability some $O(1)$-Lipschitz function $f:\R^d\to\R$ satisfying $f(x_i)=y_i$ for all $i$. Indeed, with probability $1-e^{-\Omega(d)}$ we have $||x_i-x_j||\geq 1$ for all $1\leq i\neq j\leq n$ so long as $n\leq \mathrm{poly}(d)$. In this case we may apply the Kirszbraun extension theorem to find a suitable $f$ regardless of the labels $y_i$. More explicitly we may fix a smooth bump function $g:\mathbb R^+\to\mathbb R$ with $g(0)=1$ and $g(a)=0$ for $a\geq 1$, and then interpolate using the sum of radial basis functions
\begin{equation}
\label{eq:sum-of-bumps}
  f(x)=\sum_{i=1}^n g(||x-x_i||)\cdot y_i.
\end{equation}
In fact this construction requires only $p=n(d+1)$ parameters to specify the values $(x_i,y_i)_{i\in [n]}$ and thus determine the function $f$. Hence $p=n(d+1)$ parameters suffice for robust interpolation, i.e. Theorem~\ref{thm:inf} is essentially best possible when $\mathrm{Lip}(f)=O(1)$. A similar construction shows the same conclusion for any $p\in [\widetilde{\Omega}(n),nd]$, essentially tracing the entire tradeoff curve. This is because one can first project onto a fixed subspace of dimension $\tilde d=p/n$, and the projected inputs $x_i$ now have pairwise distances at least $\Omega\left(\sqrt{\tilde d/d}\right)$ with high probability as long as $\tilde d\geq \Omega(\log n)$. The analogous construction on the projected points now requires only $p=\tilde{d}n$ parameters and has Lipschitz constant $O\left(\sqrt{d/\tilde d}\right)=O\left(\sqrt{\frac{nd}{p}}\right)$.
\end{remark}

\msedit{
\begin{remark}
  Throughout this paper we evaluate accuracy of a classifier $f$ via the sum of squared errors. In other words, we focus on the regression setting rather than classification, which is much better suited to working with Lipschitz constants. However a version of our result extends to general Lipschitz loss functions, see Corollary~\ref{cor:Rademacher}. 
\end{remark}
}

\subsection{Speculative implication for real data}
\label{subsec:speculative}
To put Theorem \ref{thm:inf} in context, we compare to the empirical results presented in \cite{madry2017towards}. In the latter work, they consider the MNIST dataset which consists of $n=6\times10^4$ images in dimension $28^2=784$. They trained robustly different architectures, and reported in Figure 4 (third plot from the left) the size of the architecture versus the obtained robust test accuracy\footnote{\msedit{A classifier $f$ is robustly accurate on input/output pair $(x,y)$ if $f(x')=y$ holds for all $x'$ in a suitable neighborhood of $x$.}}. One can see a sharp transition from roughly 10\% accuracy to roughly 90\% accuracy at around $2 \times 10^5$ parameters (capacity scale $4$ in their notation). Moreover the robust accuracy continues to increase as more parameters are added, reaching roughly 95\% accuracy at roughly $3 \times 10^6$ parameters. 
\newline

How can we compare these numbers to the law of robustness? There are a number of difficulties that we discuss below, and we emphasize that this discussion is highly speculative in nature, though we find that, with a few leaps of faith, our universal law of robustness sheds light on the potential parameter regimes of interest for robust deep learning.
\newline

The first difficulty is to evaluate the ``correct'' dimension of the problem. Certainly the number of pixels per image gives an upper bound, however one expects that the data lies on something like a lower dimensional sub-manifold. Optimistically, we hope that Theorem~\ref{thm:inf} will continue to apply for an appropriate \emph{effective dimension} which may be rather smaller than the literal number of pixels. This hope is partially justified by the fact that isoperimetry holds in many less-than-picturesque situations, some of which are stated in the next subsection. 
\newline

Estimating the effective dimension of data manifolds is an interesting problem and has attracted some study in its own right. For instance \cite{facco2017estimating, pope2021intrinsic} both predict that MNIST has effective dimension slightly larger than $10$, which is consistent with our numerical discussion at the end of this subsection. The latter also predicts an effective dimension of about $40$ for ImageNet. It is unclear how accurate these estimates are for our setting. One concrete issue is that from the point of view of isoperimetry, a ``smaller'' manifold (e.g. a sphere with radius $r<1$) will behave as though it has a larger effective dimension (e.g. $d/r^2$ instead of $d$). Thus we expect the ``scale'' of the mixture components to also be relevant for studying real datasets through our result.
\newline

Another difficulty is to estimate/interpret the noise value $\sigma^2$. From a theoretical point of view, this noise assumption is necessary for otherwise there could exist a smooth classifier with perfect accuracy in $\mathcal{F}$, defeating the point of any lower bound on the size of $\mathcal F$. We tentatively would like to think of $\sigma^2$ as capturing the contribution of the ``difficult'' part of the learning problem, that is $\sigma^2$ could be thought of as the non-robust generalization error of reasonably good models, so a couple of $\%$ of error in the case of MNIST. With that interpretation, one gets ``below the noise level'' in MNIST with a training error of a couple of $\%$. We believe that versions of the law of robustness might hold without noise; these would need to go beyond representational power and consider the dynamics of learning algorithms. 
\newline

Finally another subtlety to interpret the empirical results of \cite{madry2017towards} is that there is a mismatch between what they measure and our quantities of interest. Namely the law of robustness relates two quantities: the training error, and the worst-case robustness (i.e. the Lipschitz constant). On the other hand \cite{madry2017towards} measures the {\em robust generalization error}. Understanding the interplay between those three quantities is a fantastic open problem.
Here we take the perspective that a small robust generalization error should imply a small training error and a small Lipschitz constant. 
\newline

Another important mismatch is that we stated our universal law of robustness for Lipschitzness in $\ell_2$, while the experiments in \cite{madry2017towards} are for robustness in $\ell_{\infty}$.
We believe that a variant of the law of robustness remains true for $\ell_{\infty}$, a belief again partially justified by how broad isoperimetry is (see next subsection). 
\newline

With all the caveats described above, we can now look at the numbers as follows: in the \cite{madry2017towards} experiments, smooth models with accuracy below the noise level are attained with a number of parameters somewhere in the range $2 \times 10^5 - 3 \times 10^6$ parameters (possibly even larger depending on the interpretation of the noise level), while the law of robustness would predict any such model must have at least $n d$ parameters, and this latter quantity should be somewhere in the range $10^6 - 10^7$ (corresponding to an effective dimension between $15$ and $150$). While far from perfect, the law of robustness prediction is far more accurate than the classical rule of thumb $\#$ parameters $\simeq$ $\#$ equations (which here would predict a number of parameters of the order $10^4$).
\newline

Perhaps more interestingly, one could apply a similar reasoning to the ImageNet dataset, which consists of $1.4 \times 10^7$ images of size roughly $2 \times 10^5$. Estimating that the effective dimension is a couple of order of magnitudes smaller than this size, the law of robustness predicts that to obtain good robust models on ImageNet one would need at least $10^{10}-10^{11}$ parameters. This number is larger than the size of current neural networks trained robustly for this task, which sports between $10^8-10^9$ parameters. Thus, we arrive at the tantalizing possibility that robust models for ImageNet do not exist yet simply because we are a couple orders of magnitude off in the current scale of neural networks trained for this task.

\subsection{Related work}
\label{subsec:related}
Theorem \ref{thm:inf} is a direct follow-up to the conjectured law of robustness in \cite{bubeck2020law} for (arbitrarily weighted) two-layer neural networks with Gaussian data. Our result does not actually prove their conjecture, because we assume here polynomially bounded weights. While this assumption is reasonable from a practical perspective, it remains mathematically interesting to prove the full conjecture for the two-layer case. We prove however in Section \ref{sec:polyweights} that the polynomial weights assumption is necessary as soon as one considers three-layer neural networks. Let us also mention \cite[Theorem 6.1]{gao2019convergence} which showed a lower bound $\Omega(nd)$ on the VC dimension of any function class which can robustly interpolate \emph{arbitrary} labels on \emph{all} well-separated input sets $(x_1,\dots,x_n)$. 
\msedit{This can be viewed as a restricted version of the law of robustness for the endpoint case $L=O(1)$, where the Lipschitz constant is replaced by a robust interpolation property. Their statement and proof are of a combinatorial nature, as opposed to our probabilistic approach.} 
We also note that a relation between high-dimensional phenomenon such as concentration and adversarial examples has been hypothesized before, such as in \cite{gilmer2018adversarial}.
\newline

In addition to \cite{madry2017towards}, several recent works have experimentally studied the relationship between a neural network scale and its achieved robustness, see e.g., \cite{novak2018sensitivity,Xie2020Intriguing, gowal2021uncovering}. It has been consistently reported that larger networks help tremendously for robustness, beyond what is typically seen for classical non-robust accuracy. We view our universal law of robustness as putting this empirical observation on a more solid footing: scale is actually {\em necessary} to achieve robustness.
\newline

Another empirical thread intimately related to scale is the question of network compression, and specifically {\em knowledge distillation} \cite{hinton2015distilling}. The idea is to first train a large neural
network, and then ``distill'' it to a smaller net. It is natural to wonder whether this could be a way around the law of robustness, alas we show in Theorem \ref{thm:main} that such an approach cannot work. Indeed the latter part of Theorem \ref{thm:main} shows that the law of robustness tradeoff for the distilled net can only be improved by a logarithmic factor in the size of the original large neural network. Thus, unless one uses exponentially large networks, distillation does \textbf{not} offer a way around the law of robustness. 
\msedit{A related question is whether there might be an interaction between the number of parameters and explicit or implicit regularization, which are commonly understood to reduce effective model complexity. In our approach the number of parameters enters in bounding the covering number of $\mathcal F$ in the rather strict $L^{\infty}(\mathbb R^d;~\mathbb R)$ norm, which seems difficult to control by other means.}
\newline

The law of robustness setting is also closely related to the interpolation setting: in the former case one considers models optimizing ``beyond the noise level'', while in the latter case one studies models with perfect fit on the training data. The study of generalization in this interpolation regime has been a central focus of learning theory in the last few years (see e.g., \cite{Belkin15849, mei2020generalization, Bartlett30063, Nakkiran2020Deep}), as it seemingly contradicts classical theory about regularization. More broadly though, generalization remains a mysterious phenomon in deep learning, and the exact interplay between the law of robustness' setting (interpolation regime/worst-case robustness) and (robust) generalization error is a fantastic open problem. Interestingly, we note that one could potentially avoid the conclusion of the law of robustness (that is, that large models are necessary for robustness), with early stopping methods that could stop the optimization once the noise level is reached. In fact, this theoretically motivated suggestion has already been empirically tested and confirmed in the recent work \cite{pmlr-v119-rice20a}, showing again a close tie between the conclusions one can draw from the law of robustness and actual practical settings.
\newline

Classical lower bounds on the gradient of a function include Poincar{\'e} type inequalities, but they are of a qualitatively different nature compared to the law of robustness lower bound. We recall that a measure $\mu$ on $\R^d$ satisfies a Poincar{\'e} inequality if for any function $f$, one has $\E^{\mu}[ \|\nabla f\|^2 ] \geq C \cdot \mathrm{Var}(f)$ (for some constant $C>0$). In our context, such a lower bound for an interpolating function $f$ has essentially no consequence since the variance $f$ could be exponentially small. In fact this is tight, as one can easily use similar constructions to those in \cite{bubeck2020law} to show that one can interpolate with an exponentially small expected norm squared of the gradient (in particular it is crucial in the law of robustness to consider the Lipschitz constant, i.e., the supremum of the norm of the gradient). On the other hand, our isoperimetry assumption is related to a certain strenghtening of the Poincar{\'e} inequality known as log-Sobolov inequality (see e.g., \cite{Ledoux}). If the covariate measure satisfies only a Poincar{\'e} inequality, then we could prove a weaker law of robustness of the form  $\mathrm{Lip} \gtrsim \frac{n\sqrt{d}}{p}$ (using for example the concentration result obtained in \cite{bobkov1997poincare}). For the case of two-layer neural networks there is another natural notion of smoothness (different from $\ell_p$ norms of the gradient) that can be considered, known as the Barron norm. In \cite{BELM20} it is shown that for such a notion of smoothness there is no tradeoff {\`a} la the law of robustness, namely one can simultaneously be optimal both in terms of Barron norm and in terms of the network size. More generally, it is an interesting challenge to understand for which notions of smoothness there is a tradeoff with size.

\subsection{Isoperimetry}

Concentration of measure and isoperimetry are perhaps the most ubiquitous features of high-dimensional geometry. In short, they assert in many cases that Lipschitz functions on high-dimensional space concentrate tightly around their mean. Our result assumes that the distribution $\mu$ of the covariates $x_i$ satisfies such an inequality in the following sense.

\begin{definition}\label{defn:iso}

A probability measure $\mu$ on $\R^d$ satisfies $c$-isoperimetry if for any bounded $L$-Lipschitz $f:\R^d\to\R$, and any $t\geq 0$,
\begin{equation}
  \label{eq:iso}
  \mathbb P[|f(x)- \mathbb E[f]|\geq t]\leq 2e^{-\frac{dt^2}{2cL^2}}.
\end{equation}
\end{definition}

In general, if a scalar random variable $X$ satisfies $\Prob[|X|\geq t]\leq 2e^{-t^2/C}$ then we say $X$ is $C$-subgaussian. Hence isoperimetry states that the output of any Lipschitz function is $O(1)$-subgaussian under suitable rescaling. Distributions satisfying $O(1)$-isoperimetry include high dimensional Gaussians $\mu=\mathcal N\left(0,\frac{I_d}{d}\right)$ and uniform distributions on spheres and hypercubes (normalized to have diameter $1$). Isoperimetry also holds for mild perturbations of these idealized scenarios, including\footnote{The first two examples satisfy a logarithmic Sobolev inequality, which implies isoperimetry \cite[Proposition 2.3]{ledoux1999concentration}.}:

\begin{itemize}
    \item The sum of a Gaussian and an independent random vector of small norm \cite{chen2021dimension}.
    \item Strongly log-concave measures in any normed space \cite[Proposition 3.1]{bobkov2000brunn}.
    \item Manifolds with positive Ricci curvature \cite[Theorem 2.2]{gromov1986isoperimetric}.
\end{itemize}

Due to the last condition above, we believe our results are realistic even under the \emph{manifold hypothesis} that high-dimensional data tends to lie on a lower-dimensional submanifold \msedit{(which may be difficult to describe cleanly with coordinates). Recalling the discussion of Subsection~\ref{subsec:speculative}, \cite[Theorem 2.2]{gromov1986isoperimetric} implies that for submanifolds $M\subseteq\mathbb R$ with Ricci curvature $\Omega(\dim(M))$ uniformly\footnote{This is the natural scaling because the Ricci curvature of $M$ can be defined by summing its sectional curvatures on $\dim(M)$ two-dimensional subspaces.}, the law of robustness provably holds relative to the intrinsic dimension $\dim(M)$.} This viewpoint on learning has been studied for decades, see e.g. \cite{hastie1989principal,kambhatla1993fast,roweis2000nonlinear,tenenbaum2000global,narayanan2010sample,fefferman2016testing}.
We also note that our formal theorem (Theorem \ref{thm:main}) actually applies to distributions that can be written as a mixture of distributions satisfying isoperimetry. Let us also point out that from a technical perspective, our proof is not tied to the Euclidean norm and applies essentially whenever Definition~\ref{defn:iso} holds. The main difficulty in extending the law of robustness to e.g. the earth-mover distance seems to be identifying realistic cases which satisfy isoperimetry.
\newline

Our proofs will repeatedly use the following simple fact:
\begin{proposition}
\label{prop:versh}
If $X_1,\dots,X_n$ are independent and $C$-subgaussian, with mean $0$, then $X_{av}=\frac{1}{\sqrt{n}}\sum_{i=1}^n X_i$ is $18C$-subgaussian. 
\end{proposition}

\begin{proof}
By \cite[Exercise 3.1 part d.]{van2014probability}, 
\[
  \mathbb E\left[e^{X_i^2/3C}\right]\leq 2,\quad i\in [n].
\]
It is immediate by H{\"o}lder that the same bound holds for $X_{av}$ in place of $X_i$, and using \cite[Exercise 3.1 parts e. and c.]{van2014probability} now implies the first claim. The second claim follows similarly, since by convexity we have
\[
  \mathbb E[e^{Y^2/3C}]\leq \mathbb E[e^{X_1^2/3C}]\leq 2.
\]
\end{proof}

\section{A finite approach to the law of robustness}
For the function class of two-layer neural networks, \cite{bubeck2020law} investigated several approaches to prove the law of robustness. At a high level, the proof strategies there relied on various ways to measure how ``large'' the set of  two-layer neural networks can be (specifically, they tried a geometric approach based on relating to multi-index models, a statistical approach based on the Rademacher complexity, and an algebraic approach for the case of polynomial activations).
\newline

In this work we take here a different route: we shift the focus from the function class $\mathcal F$ to an {\em individual} function $f \in \mathcal F$. Namely, our proof starts by asking the following question: for a fixed function $f$, what is the probability that it would give a good approximate fit on the (random) data? For simplicity, consider for a moment the case where we require $f$ to actually interpolate the data (i.e., perfect fit), and say that $y_i$ are random $\pm1$ labels. The key insight is that isoperimetry implies that {\em either} the $0$-level set of $f$ {\em or} the $1$-level set of $f$ must have probability smaller than $\exp\left(- \frac{d}{\mathrm{Lip}(f)^2} \right)$. Thus, the probability that $f$ fits all the $n$ points is at most $\exp\left(- \frac{n d}{\mathrm{Lip}(f)^2} \right)$ so long as both labels $y_i\in \{-1,1\}$ actually appear a constant fraction of the time. In particular, using an union bound\footnote{In this informal argument we ignore the possibility that the labels $y_i$ are not well-balanced. Note that the probability of this rare event is not amplified by a union bound over $f\in\mathcal F$.}, for a finite function class $\mathcal F$ of size $N$ with $L$-Lipschitz functions, the probability that there exists a function $f \in \mathcal F$ fitting the data is at most 
\[
N \exp\left(- \frac{n d}{L^2} \right) = \exp\left(\log(N) - \frac{n d}{L^2} \right) \,.
\]
Thus we see that, if $L \ll \sqrt{\frac{n d}{\log(N)}}$, then the probability of finding a fitting function in $\mathcal F$ is very small. This basically concludes the proof, since via a standard discretization argument, for a smoothly parametrized family with $p$ (bounded) parameters one expects $\log(N) = \tilde{O}(p)$.
\newline

We now give the formal proof, which applies in particular to approximate fit rather than exact fit in the argument above. The only difference is that we will identify a well-chosen subgaussian random variable in the problem. 
We start with the finite function class case:

\begin{theorem} \label{thm:finite}
Let $(x_i,y_i)_{i\in [n]}$ be i.i.d. input-output pairs in $\R^d \times [-1,1]$ such that:
\begin{enumerate}
    \item The distribution $\mu$ of the covariates $x_i$ can be written as $\mu = \sum_{\ell=1}^k \alpha_{\ell} \mu_{\ell}$, where each $\mu_{\ell}$ satisfies $c$-isoperimetry and $\alpha_{\ell} \geq 0, \sum_{\ell=1}^k \alpha_{\ell}=1$.
    \item The expected conditional variance of the output is strictly positive, denoted $\sigma^2 \equiv \E^{\mu}[Var[y|x]] > 0$.
\end{enumerate}   
Then one has: 
\begin{align*}
  & \mathbb{P} \left( \exists f \in \mathcal F : \frac{1}{n} \sum_{i=1}^n (y_i - f(x_i))^2 \leq \sigma^2 - \epsilon \right) 
  \\
  & \leq 4 k \exp\left(- \frac{n \epsilon^2}{8^3 k}  \right) 
  +
   2 \exp\left(\log(|\mathcal F|) - \frac{\epsilon^2 n d}{10^4 c L^2} \right)  \,. 
\end{align*}
\end{theorem}

We start with a lemma showing that, to optimize beyond the noise level one must necessarily correlate with the noise part of the labels. Below and throughout the rest of the paper we write
\begin{align*}
g(x) &= \E[y | x],
\\
z_i &= y_i - g(x_i)
\end{align*}
for the target function, and for the noise part of the observed labels, respectively. (In particular $y_i$ is the sum of the target function $g(x_i)$ and the noise term $z_i$.) 

\begin{lemma} \label{lem:corr}
One has
\[
\mathbb{P} \left( \exists f \in \mathcal F : \frac{1}{n} \sum_{i=1}^n (y_i - f(x_i))^2 \leq \sigma^2 - \epsilon \right) \leq 2 \exp\left(- \frac{n \epsilon^2}{8^3}  \right)  + \mathbb{P} \left( \exists f \in \mathcal F : \frac{1}{n} \sum_{i=1}^n f(x_i) z_i \geq \frac{\epsilon}{4}  \right) \,.
\]
\end{lemma}

\begin{proof}
The sequence $(z_i^2)$ is i.i.d., with mean $\sigma^2$, and such that $|z_i|^2 \leq 4$. Thus Hoeffding's inequality yields:
\begin{equation} \label{eq:conc1}
\mathbb{P} \left( \frac{1}{n} \sum_{i=1}^n z_i^2 \leq \sigma^2 - \frac{\epsilon}{6} \right) \leq \exp\left(- \frac{n \epsilon^2}{8^3}  \right) \,.
\end{equation}
On the other hand the sequence $(z_i g(x_i))$ is i.i.d., with mean $0$ (since $\E[z_i | x_i] = 0$), and such that $|z_i g(x_i)| \leq 2$. Thus Hoeffding's inequality yields:
\begin{equation} \label{eq:conc2}
\mathbb{P} \left( \frac{1}{n} \sum_{i=1}^n z_i g(x_i) \leq - \frac{\epsilon}{6}  \right) \leq \exp\left(- \frac{n \epsilon^2}{8^3}  \right) \,.
\end{equation}
Let us write $Z= \frac{1}{\sqrt{n}}(z_1,\hdots,z_n), G= \frac{1}{\sqrt{n}}(g(x_1),\hdots, g(x_n))$, and $F= \frac{1}{\sqrt{n}}(f(x_1),\hdots,f(x_n))$. We claim that if $\|Z\|^2 \geq \sigma^2 - \frac{\epsilon}{6}$ and $\langle Z, G\rangle \geq - \frac{\epsilon}{6}$, then for any $f \in \mathcal F$ one has
\[
\|G + Z - F\|^2 \leq \sigma^2 - \epsilon \, \Rightarrow \,
\langle F, Z \rangle \geq \frac{\epsilon}{4} \,. 
\]
This claim together with \eqref{eq:conc1} and \eqref{eq:conc2} conclude the proof. On the other hand the claim itself directly follows from:
\[
\sigma^2 - \epsilon \geq \|G + Z - F\|^2 = \|Z + G - F\|^2 = \|Z\|^2 + 2 \langle Z, G-F\rangle + \|G- F\|^2 \geq \sigma^2 - \frac{\epsilon}{2} - 2 \langle Z, F \rangle \,.
\]
\end{proof}

We can now proceed to the proof of Theorem \ref{thm:finite}:
\begin{proof}
First note that without loss of generality we can assume that the range of any function in $\mathcal F$ is included in $[-1,1]$ (indeed clipping the values improves both the fit to any $y \in [-1,1]$ and the Lipschitz constant). We also assume without loss of generality that all functions in $\mathcal F$ are $L$-Lipschitz.

For clarity let us start with the case $k=1$. By the isoperimetry assumption we have that $\sqrt{\frac{d}{c}}\frac{f(x_i)-\mathbb{E}[f]}{L}$ is $2$-subgaussian. Since $|z_i| \leq 2$, we also have that $\sqrt{\frac{d}{c}}\frac{(f(x_i)-\E[f]) z_i}{L}$ is $8$-subgaussian. Moreover, the latter random variable has mean zero since $\E[z|x] = 0$. Thus by Proposition \ref{prop:versh} (and $8\times 18=12^2)$ we have:
\[
  \mathbb{P} \left( \sqrt{\frac{d}{c n L^2}} 
  \sum_{i=1}^n 
  (f(x_i)-\E[f]) z_i \geq t \right) 
  \leq 
  2 \exp\left( - (t/12)^2 \right) \,.
\]
Rewriting (and noting $12\times 8\leq 10^2$), we find:
\begin{equation} 
\label{eq:conc3}
  \mathbb{P} \left(\frac{1}{n} \sum_{i=1}^n (f(x_i)-\E[f]) z_i \geq \frac{\epsilon}{8} \right) 
  \leq 
  2 \exp\left( - \frac{\epsilon^2 n d}{10^4 c L^2} \right) \,.
\end{equation}
Since we assumed that the range of the functions is in $[-1,1]$ we have $\E[f] \in [-1,1]$ and hence:
\begin{equation} \label{eq:conc4}
\mathbb{P} \left(\exists f \in \mathcal F : \frac{1}{n} \sum_{i=1}^n \E[f] z_i \geq \frac{\epsilon}{8} \right) \leq \mathbb{P} \left(\left| \frac{1}{n} \sum_{i =1}^n z_i \right| \geq \frac{\epsilon}{8} \right) \,.
\end{equation} 
(This step is the analog of requiring the labels $y_i$ to be well-balanced in the example of perfect interpolation.) By Hoeffding's inequality, the above quantity is smaller than $2 \exp( - n \epsilon^2 / 8^3)$ (recall that $|z_i| \leq 2$). Thus we obtain with a union bound:
\begin{eqnarray*}
\mathbb{P} \left(\exists f \in \mathcal F : \frac{1}{n} \sum_{i=1}^n f(x_i) z_i \geq \frac{\epsilon}{4} \right) & \leq & |\mathcal F| \cdot \mathbb{P} \left(\frac{1}{n} \sum_{i=1}^n (f(x_i)-\E[f]) z_i \geq \frac{\epsilon}{8} \right) + \mathbb{P} \left(\left| \frac{1}{n} \sum_{i =1}^n z_i \right| \geq \frac{\epsilon}{8} \right) \\
& \leq & 2 |\mathcal F| \cdot \exp\left( - \frac{\epsilon^2 n d}{10^4 c L^2} \right) + 2 \exp\left( - \frac{n \epsilon^2}{8^3} \right)  \,.
\end{eqnarray*}
Together with Lemma \ref{lem:corr} this concludes the proof for $k=1$.
\newline

We now turn to the case $k >1$. We first sample the mixture component $\ell_i \in [k]$ for each data point $i \in [n]$, and we now reason conditioned on these mixture components. Let $S_{\ell} \subset [n]$ be the set of data points sampled from mixture component $\ell \in [k]$, that is $x_i, i \in S_{\ell},$ is i.i.d. from $\mu_{\ell}$. We now have that $\sqrt{\frac{d}{c}}\frac{f(x_i)-\mathbb{E}^{\mu_{\ell_i}}[f]}{L}$ is $1$-subgaussian (notice that the only difference is that now we need to center by $\mathbb{E}^{\mu_{\ell_i}}[f]$, which depends on the mixture component). In particular using the same reasoning as for \eqref{eq:conc3} we obtain (crucially note that Proposition \ref{prop:versh} does not require the random variables to be identically distributed):
\begin{equation} 
\label{eq:conc5}
  \mathbb{P} \left(\frac{1}{n} \sum_{i=1}^n (f(x_i)-\E^{\mu_{\ell_i}}[f]) z_i \geq \frac{\epsilon}{8} \right) 
  \leq 
  2 \exp\left( - \frac{\epsilon^2 n d}{9^4 c L^2} \right) \,.
\end{equation}
Next we want to appropriately modify \eqref{eq:conc4}. To do so note that:
\[
\max_{m_1, \hdots, m_k \in [-1,1]} \sum_{i=1}^n m_{\ell_i} z_i = \sum_{\ell=1}^k \left| \sum_{i \in S_{\ell}} z_i \right| \,,
\] 
so that we can rewrite \eqref{eq:conc4} as:
\[
\mathbb{P} \left(\exists f \in \mathcal F : \frac{1}{n} \sum_{i=1}^n \E^{\mu_{\ell_i}}[f] z_i \geq \frac{\epsilon}{8} \right) \leq  \mathbb{P} \left(\frac{1}{n}  \sum_{\ell=1}^k \left| \sum_{i \in S_{\ell}} z_i \right| \geq \frac{\epsilon}{8} \right) \,. 
\]
Now note that $\sum_{\ell=1}^k \sqrt{|S_{\ell}|} \leq \sqrt{n k}$ and thus we have:
\[
  \mathbb{P} \left(\frac{1}{n}  \sum_{\ell=1}^k \left| \sum_{i \in S_{\ell}} z_i \right| \geq \frac{\epsilon}{8} \right) 
 \leq   \mathbb{P} \left(\sum_{\ell=1}^k \left| \sum_{i \in S_{\ell}} z_i \right| \geq \frac{\epsilon}{8} \sqrt{\frac{n}{k}} \sum_{\ell=1}^k \sqrt{|S_{\ell}|} \right) \leq \sum_{\ell=1}^k \mathbb{P} \left( \left| \sum_{i \in S_{\ell}} z_i \right| \geq \frac{\epsilon}{8} \sqrt{\frac{n}{k}} \sqrt{|S_{\ell}|} \right) \,.
\]
Finally by Hoeffding's inequality, we have for any $\ell \in [k]$, $\mathbb{P} \left( \left| \sum_{i \in S_{\ell}} z_i \right| \geq t \sqrt{|S_{\ell}|} \right) \leq 2 \exp \left( - \frac{t^2}{8} \right)$, and thus the last display is bounded from above by $2 k \exp \left( - \frac{n \epsilon^2}{8^3 k} \right)$. The proof can now be concluded as in the case $k=1$.
\end{proof}

In fact the above result can be further improved for small $\sigma$ using the following Lemma~\ref{lem:sigma-improve-main}. Note that the additional assumption on $d$ is rather mild because it is required for the latter term to be smaller than $|\mathcal F|e^{-O(n)}$. (In particular, we are primarily interested in the regime of large $n,d$ and constant $\sigma,\epsilon,c$.)

\msedit{
\begin{lemma}
\label{lem:sigma-improve-main}
There exist absolute constants $C_1,C_2$ such that the following holds. In the setting of Theorem~\ref{thm:finite}, assume $d\geq C_1\cdot\left(\frac{cL^2\sigma^2}{\epsilon^2}\right)$. Then
\[
  \mathbb{P} \left(
  \exists f \in \mathcal F
  :
  \frac{1}{n} \sum_{i=1}^n (f(x_i)-\E^{\mu_{\ell_i}}[f]) z_i \geq \frac{\epsilon}{8}
  \right) 
  \leq 
  \exp\left(-\frac{n \sigma^4}{8}\right)
  +
  \exp\left(\log|\mathcal F|-\frac{\epsilon^2nd}{C_1 cL^2\sigma^2}\right)
  .
\]
\end{lemma}

\begin{proof}
  We use the simple estimate
  \begin{equation}
  \label{eq:CS}
  \sup_{f\in\mathcal F}\left|\sum_{i=1}^n (f(x_i)-\E^{\mu_{\ell_i}}[f]) z_i\right|
  \leq
  \sqrt{\sup_{f\in\mathcal F}\sum_{i=1}^n(f(x_i)-\E^{\mu_{\ell_i}}[f])^2}\times\sqrt{\sum_{i=1}^n z_i^2}.
  \end{equation}
  Applying Hoeffding's inequality as in \eqref{eq:conc1} yields
  \begin{equation}
  \label{eq:zi2-bound}
  \mathbb P\left[\sum_{i=1}^n z_i^2 \geq 2\sigma^2 n\right]
  \leq
  \exp\left(-\frac{n\sigma^4}{8}\right).
  \end{equation}
  Next we upper bound the tail of $\sum_{i=1}^n(f(x_i)-\E^{\mu_{\ell_i}}[f])^2$ for each fixed $f$. Since 
  \[
  f(x_i)-\E^{\mu_{\ell_i}}[f]
  \]
  is sub-Gaussian, it follows that its square is sub-exponential, i.e. (recall \cite[Definition 2.7.5]{vershynin2018high})
  \[
  \|(f(x_i)-\E^{\mu_{\ell_i}}[f])^2\|_{\psi_1}\leq O(cL^2/d).
  \]
  Let 
  \[
    W_i=(f(x_i)-\E^{\mu_{\ell_i}}[f])^2-\E\left[(f(x_i)-\E^{\mu_{\ell_i}}[f])^2\right]
  \]
  and note that 
  \begin{equation}
  \label{eq:squared-expectation-bound}
  0\leq \E\left[(f(x_i)-\E^{\mu_{\ell_i}}[f])^2\right]\leq O(cL^2/d).
  \end{equation}
  As centering decreases the sub-exponential norm (\cite[Exercise 2.7.10]{vershynin2018high}), we have 
  \[
  \left\|W_i\right\|_{\psi_1}\leq O(cL^2/d)
  \]
  Note that for $d\geq \frac{2^8 cL^2\sigma^2}{\epsilon^2}$ (which is ensured for a large constant $C_1$ in the hypothesis) we have
  \[
  \min\left(\frac{\left(\frac{n\epsilon^2}{2^8\sigma^2}\right)^2}{n(cL^2/d)^2},\frac{\frac{n\epsilon^2}{2^8\sigma^2}}{cL^2/d}\right)
  =
  \min\left(\frac{\epsilon^4 nd^2}{2^{16} (cL^2)^2\sigma^4},\frac{\epsilon^2 nd}{2^8 cL^2\sigma^2}\right)
  =
  \frac{\epsilon^2 nd}{2^8 cL^2\sigma^2}.
  \]
  Hence Bernstein's inequality (e.g. \cite[Theorem 2.8.1]{vershynin2018high}) implies
  \[
  \mathbb P\left[\sum_{i=1}^n W_i\geq \frac{n\epsilon^2}{2^8\sigma^2}\right]
  \leq 
  2\exp\left(-\Omega\left(\frac{\epsilon^2 nd}{cL^2\sigma^2}\right)\right).
  \]
  Recalling \eqref{eq:squared-expectation-bound} and union bounding over $f\in\mathcal F$, we find
  \begin{equation}
  \label{eq:Bernstein-bound}
  \begin{aligned}
  \mathbb P\left[\sup_{f\in\mathcal F}\sum_{i=1}^n (f(x_i)-\E^{\mu_{\ell_i}}[f])^2\geq \frac{n\epsilon^2}{2^7\sigma^2}\right]
  &\leq
  |\mathcal F|\cdot\sup_{f\in\mathcal F}\mathbb P\left[\sum_{i=1}^n (f(x_i)-\E^{\mu_{\ell_i}}[f])^2\geq O\left(\frac{cL^2n}{d}\right)+\frac{n\epsilon^2}{2^8\sigma^2}\right]
  \\
  &\leq
  2|\mathcal F|\exp\left(-\Omega\left(\frac{\epsilon^2 nd}{cL^2\sigma^2}\right)\right).
  \end{aligned}
  \end{equation}
  (Here we again used the assumed lower bound on $d$.)
  Finally on the event that both
  \begin{align*}
  \sup_{f\in\mathcal F}\sum_{i=1}^n (f(x_i)-\E^{\mu_{\ell_i}}[f])^2&\leq \frac{n\epsilon^2}{2^7\sigma^2},
  \\
  \sum_{i=1}^n z_i^2 &\leq 2\sigma^2 n
  \end{align*}
  hold, applying \eqref{eq:CS} yields
  \[
  \sup_{f\in\mathcal F}\left|\sum_{i=1}^n (f(x_i)-\E^{\mu_{\ell_i}}[f]) z_i\right|
  \leq
  \sqrt{\frac{n\epsilon^2}{2^7\sigma^2}}\times\sqrt{2\sigma^2 n}
  \leq
  \frac{n\epsilon}{8}.
  \]
  Combining \eqref{eq:zi2-bound} with \eqref{eq:Bernstein-bound} now completes the proof.
\end{proof}

By using Lemma~\ref{lem:sigma-improve-main} in place of \eqref{eq:conc5} when proving Theorem~\ref{thm:finite}, one readily obtains the following.

\begin{theorem}
\label{thm:sigma-improve}
There exist absolute constants $C_1,C_2$ such that the following holds. In the setting of Theorem~\ref{thm:finite}, assume $d\geq C_1\cdot\left(\frac{cL^2\sigma^2}{\epsilon^2}\right)$. Then
\[
  \mathbb{P} \left( \exists f \in \mathcal F : \frac{1}{n} \sum_{i=1}^n (y_i - f(x_i))^2 \leq \sigma^2 - \epsilon \right) 
  \leq
  (4k+1) \exp\left(- \frac{n \epsilon^2}{8^3 k}  \right) 
  +
  \exp\left(\log|\mathcal F|-\frac{\epsilon^2nd}{C_2cL^2\sigma^2}\right)
  .
\]
\end{theorem}

\begin{proof}
Using Lemma~\ref{lem:sigma-improve-main} in place of \eqref{eq:conc5} when proving Theorem~\ref{thm:finite} immediately implies
\[
  \mathbb{P} \left( \exists f \in \mathcal F : \frac{1}{n} \sum_{i=1}^n (y_i - f(x_i))^2 \leq \sigma^2 - \epsilon \right) 
  \leq
  4k \exp\left(- \frac{n \epsilon^2}{8^3 k}  \right) 
  +
  \exp\left(-\frac{n \sigma^4}{8}\right)
  +
  \exp\left(\log|\mathcal F|-\frac{\epsilon^2nd}{ C_2 cL^2\sigma^2}\right)
  .
\]
It remains to observe that $\frac{\epsilon^2}{8^3k}\leq \frac{\sigma^4}{8}$ since $\epsilon\leq \sigma^2$.
\end{proof}

Finally we can now state and prove the formal version of the informal Theorem~\ref{thm:inf} from the introduction. 

\begin{theorem}
\label{thm:main}
Let $\mathcal F$ be a class of functions from $\R^d\to \R$ and let $(x_i,y_i)_{i\in [n]}$ be i.i.d. input-output pairs in $\R^d\times [-1,1]$. Fix $\epsilon, \delta \in (0,1)$. Assume that:
\begin{enumerate}
    \item The function class can be written as $\mathcal F = \{ f_{\bw}, \bw \in \mathcal W\}$ with $\mathcal{W} \subset \R^p$, $\mathrm{diam}(\mathcal W) \leq W$ and for any $\bw_1, \bw_2 \in \mathcal W$,
    \[
      ||f_{\bw_1}-f_{\bw_2}||_{\infty}\leq J||\bw_1-\bw_2||.
    \]
    \item The distribution $\mu$ of the covariates $x_i$ can be written as $\mu = \sum_{\ell=1}^k \alpha_{\ell} \mu_{\ell}$, where each $\mu_{\ell}$ satisfies $c$-isoperimetry, $\alpha_{\ell} \geq 0, \sum_{\ell=1}^k \alpha_{\ell}=1$, and $k$ is such that 
    \begin{equation}
    \label{eq:k-bound}
    10^4 k \log(8k / \delta) \leq n \epsilon^2.
    \end{equation}
    \item The expected conditional variance of the output is strictly positive, denoted $\sigma^2 \equiv \E^{\mu}[Var[y|x]]>0$.
    \item The dimension $d$ is large compared to $\epsilon$: 
    \begin{equation}
    \label{eq:d-bound}
      d\geq C_1\left(\frac{cL^2\sigma^2}{\epsilon^2}\right).
    \end{equation} 
\end{enumerate}    
Then, with probability at least $1-\delta$ with respect to the sampling of the data, one has simultaneously for all $f \in \mathcal F$:
\begin{equation}
\label{eq:main}
  \frac{1}{n}\sum_{i=1}^n (f(x_i)-y_i)^2 \leq \sigma^2-\epsilon \,\, \Rightarrow \,\, \mathrm{Lip}(f) \geq \frac{\epsilon}{\sigma \sqrt{C_2 c}}\times \sqrt{\frac{nd}{p \log(1+60 W J \epsilon^{-1}) + \log(4/\delta)}}  \,.
\end{equation}
Moreover if $\mathcal W$ consists only of $s$-sparse vectors with $||w||_0\leq s$, then the above inequality improves to 
\begin{equation}
\label{eq:sparse}
  \frac{1}{n}\sum_{i=1}^n (f(x_i)-y_i)^2 \leq \sigma^2-\epsilon \,\, \Rightarrow \,\, \mathrm{Lip}(f) \geq \frac{\epsilon}{\sigma \sqrt{C_2 c}} \sqrt{\frac{nd}{s \log\big(p(1+60 W J  \epsilon^{-1})\big) + \log(4/\delta)}}  \,.
\end{equation}
\end{theorem}

Note that as in the previous lemmas, Theorem~\ref{thm:main} requires the dimension $d$ to be at least a constant depending on $\epsilon$ in \eqref{eq:d-bound}. This extra condition is unnecessary if one uses Theorem~\ref{thm:finite} in place of Theorem~\ref{thm:sigma-improve} (which would sacrifice a factor $\sigma$ in the resulting lower bound on $\mathrm{Lip}(f)$).

\begin{proof}[Proof of Theorem~\ref{thm:main}]
Define $\mathcal{W}_L\subseteq\mathcal W$ by 
\[
  \mathcal{W}_L\equiv\{\bw\in \mathcal W:\mathrm{Lip}(f_{\bw})\leq L\}.
\]
Denote $\mathcal{W}_{L,\epsilon}\subseteq \mathcal{W}_L$ for an $\frac{\epsilon}{8J}$-net of $\mathcal{W}_L$. We have in particular $|\mathcal W_{\epsilon}|\leq (1+60WJ\epsilon^{-1})^{p}$ (see e.g. \cite[Corollary 4.2.13]{vershynin2018high}). We apply Theorem \ref{thm:sigma-improve} to $\mathcal F_{L,\epsilon} \equiv \{f_{\bw}, \bw \in \mathcal{W}_{L,\epsilon}\}$:
\begin{align*}
  & \mathbb{P} \left( \exists f \in \mathcal F_{L,\epsilon}  : \frac{1}{n} \sum_{i=1}^n (y_i - f(x_i))^2 \leq \sigma^2 - \frac{\epsilon}{2}\right) 
  \\
  & \leq 
  (4 k+1) \exp\left(- \frac{n \epsilon^2}{9^4 k}  \right) 
  +
  \exp\left(p \log(1+60WJ\epsilon^{-1}) - \Omega\left(\frac{\epsilon^2 n d}{c L^2 \sigma^2} \right)\right)  \,. 
\end{align*}
Observe that if $\|f-g\|_{\infty} \leq \frac{\epsilon}{8}$ and $\|y\|_{\infty},\|f\|_{\infty},\|g\|_{\infty}\leq 1$, then 
\[
  \frac{1}{n} \sum_{i=1}^n (y_i - f(x_i))^2 \leq \frac{\epsilon}{2} + \frac{1}{n} \sum_{i=1}^n (y_i - g(x_i))^2.
\]
(We may again assume without loss of generality that all functions in $\mathcal F$ map to $[-1,1]$.) Thus we obtain for any $L>0$ and an absolute constant $C_1$
\begin{align}
\label{eq:prob-just-above}
  & \mathbb{P} \left( \exists f \in \mathcal F  : \frac{1}{n} \sum_{i=1}^n (y_i - f(x_i))^2 \leq \sigma^2 - \epsilon \text{ and } \mathrm{Lip}(f) \leq L \right) 
  \\
\nonumber
  &\leq
  (4 k+1) \exp\left(- \frac{n \epsilon^2}{10^4 k}  \right) 
  +
  \exp\left(p \log(1+60WJ\epsilon^{-1}) - \frac{\epsilon^2 n d}{C_1c L^2 \sigma^2} \right)  \,. 
\end{align}
The first assumption ensures that for any $\bw\in \mathcal{W}_L$, there is $\bw'\in\mathcal{W}_{L,\epsilon}$ with $\|f_{\bw}-f_{\bw'}\|_{\infty}\leq\frac{\epsilon}{8}$. =
Finally we use the second assumption to show the probability in \eqref{eq:prob-just-above} just above is at most $\delta$ if
\[
  L\leq\frac{\epsilon}{C_2\sigma \sqrt{c}} \sqrt{\frac{nd}{p \log(1+60 W J \epsilon^{-1}) + \log(4/\delta)}}
\]
for a large absolute constant $C_2$. The first term is estimated (recall \eqref{eq:k-bound}) via
\[
  (4 k+1) \exp\left(- \frac{n \epsilon^2}{10^4 k}  \right) \leq \frac{(4k+1)\delta}{8k}\leq \frac{3\delta}{4}.
\]
The second term is estimated by
\[
  \exp\left(p \log(1+60WJ\epsilon^{-1}) - \frac{\epsilon^2 n d}{C_2 c L^2 \sigma^2} \right) 
  \leq
  e^{-\log(4/\delta)}
  =
  \frac{\delta}{4}
\]
Combining these estimates on \eqref{eq:prob-just-above} proves \eqref{eq:main}.

To show \eqref{eq:sparse}, the proof proceeds identically after the improved estimate $|\mathcal W_{\epsilon}|\leq \big(p(1+60WJ\epsilon^{-1})\big)^{s}$. To obtain this estimate, note that the number of $s$-subsets $S\subseteq \binom{[p]}{s}$ is at most $p^s$. Letting $\mathcal W_S$ consist of those $w\in\mathcal W$ with $w_i=0$ for all $i\notin S$, the size of an $\epsilon$-net $\mathcal W_{S,\epsilon}$ for $\mathcal W_S$ is $|\mathcal W_{S,\epsilon}|\leq (1+60WJ\epsilon^{-1})^{s}$. Therefore the union 
\[
  \bigcup_{S\subseteq \binom{[p]}{s}} \mathcal W_{S,\epsilon}
\]
is an $\epsilon$-net of $\mathcal W$ of size at most $\big(p(1+60WJ\epsilon^{-1})\big)^{s}$ as claimed above.
\end{proof}
}

\section{Deep neural networks} \label{sec:continuous}
We now specialize the law of robustness (Theorem \ref{thm:main}) to multi-layer neural networks. We consider a rather general class of depth $D$ neural networks described as follows. First, we require that the neurons are partitioned into layers $\mathcal L_1,\dots,\mathcal L_D$, and that all connections are from $\mathcal L_i\to\mathcal L_j$ for some $i<j$. This includes the basic feed-forward case in which only connections $\mathcal L_i\to\mathcal L_{i+1}$ are used as well as more general skip connections. We specify (in the natural way) a neural network by matrices $W_j$ of shape $|\mathcal L_j|\times \sum_{i<j}|\mathcal L_i|$ for each $1\leq j \leq D$, as well as $1$-Lipschitz non-linearities $\sigma_{j,\ell}$ and scalar biases $b_{j,\ell}$ for each $(j,\ell)$ satisfying $\ell\in |\mathcal L_j|$. We use fixed non-linearities $\sigma_{j,\ell}$ as well as a fixed architecture, in the sense that each matrix entry $W_{j}[k,\ell]$ is either always $0$ or else it is variable (and similarly for the bias terms). 
\newline

To match the notation of Theorem \ref{thm:main}, we identify the parametrization in terms of the matrices $(W_{j})$ and bias terms $(b_{j,\ell})$ to a single $p$-dimensional vector $\bw$ as follows. A variable matrix entry $W_{j}[k,\ell]$ is set to $w_{a(j,k,\ell)}$ for some fixed index $a(j,k,\ell)\in [p]$, and a variable bias term $b_{j,\ell}$ is set to $w_{a(j,\ell)}$ for some $a(j,\ell)\in [p]$. Thus we now have a parametrization $\bw \in \R^p \mapsto f_{\bw}$ where $f_{\bw}$ is the neural network represented by the parameter vector $\bw$. Importantly, note that our formulation allows for weight sharing (in the sense that a shared weight is counted only as a single parameter). For example, this is important to obtain an accurate count of the number of parameters in convolutional architectures.
\newline

In order to apply Theorem \ref{thm:main} to this class of functions we need to estimate the Lipschitz constant of the parametrization $\bw \mapsto f_{\bw}$. To do this we introduce three more quantities. First, we shall assume that all the parameters are bounded in magnitude by $W$, that is we consider the set of neural networks parametrized by $\bw \in [-W,W]^p$. Next, for the architecture under consideration, denote $Q$ for the maximum number of matrix entries/bias terms that are tied to a single parameter $w_a$ for some $a \in [p]$. Finally we define
\[
  B(\bw) =  \prod_{j\in [D]} \max(\|W_{j}\|_{op},1).
\]
Observe that $B(\bw)$ is an upper bound on the Lipschitz constant of the network itself, i.e., the map $x \mapsto f_{\bw}(x)$. It turns out that a uniform control on it also controls the Lipschitz constant of the {\em parametrization} $\bw \mapsto f_{\bw}$. Namely we have the following lemma:
\begin{lemma}\label{lem:NNlip}
Let $x \in \R^d$ such that $\|x\| \leq R$, and $\bw_1, \bw_2 \in \R^p$ such that $B({\bw_1}),B({\bw_2})\leq\overline{B}$. Then one has
\[
|f_{\bw_1}(x)-f_{\bw_2}(x)|\leq \overline{B}^2Q R\sqrt{p} \|\bw_1-\bw_2\| \,.
\]
Moreover for any $\bw \in [-W,W]^p$ with $W\geq 1$, one has
\[
B(\bw) \leq (W \sqrt{pQ})^D.
\]
\end{lemma}

\begin{proof}
Fix an input $x$ and define $g_x$ by $g_x(\bw)= f_{\bw}(x)$. A standard gradient calculation for multi-layer neural networks directly shows that $\|\nabla g_x(\bw)\|_{\infty} \leq B(\bw) Q R$
so that $\|\nabla g_x(\bw)\| \leq B(\bw) Q R \sqrt{p}$. Since the matrix operator norm is convex (and nonnegative) it follows that $B(\bw) \leq B(\bw_1)B(\bw_2)\leq \overline{B}^2$ on the entire segment $[\bw_1,\bw_2]$ by multiplying over layers. Thus  $\|\nabla g_x(\bw)\| \leq \overline{B}^2 Q R \sqrt{p}$ on that segment, which concludes the proof of the first claimed inequality. The second claimed inequality follows directly from $\|W_{j}\|_{op}\leq \|W_{j}\|_2 \leq W\sqrt{p Q}$.
\end{proof}

Lemma \ref{lem:NNlip} shows that when applying Theorem \ref{thm:main} to our class of neural networks one can always take $J= R (W Q p)^{D}$ (assuming that the covariate measure $\mu$ is supported on the ball of radius $R$). Thus in this case the law of robustness (under the assumptions of Theorem \ref{thm:main}) directly states that with high probability, any neural network in our class that fits the training data well below the noise level must also have:
\begin{equation} 
\label{eq:withdepth}
  \mathrm{Lip}(f) \geq \tilde{\Omega} \left(\sqrt{\frac{n d}{D p}} \right) \,,
\end{equation}
where $\tilde{\Omega}$ hides logarithmic factors in $W, p, R, Q$, and the probability of error $\delta$. Thus we see that the law of robustness, namely that the number of parameters should be at least $n d$ for a smooth model with low training error, remains intact for constant depth neural networks. If taken at face value, the lower bound \eqref{eq:withdepth} suggests that it is better in practice to distribute the parameters towards {\em depth} rather than {\em width}, since the lower bound is decreasing with $D$. On the other hand, we note that \eqref{eq:withdepth} can be strengthened to:
\begin{equation} 
\label{eq:withoutdepth}
  \mathrm{Lip}(f) \geq \tilde{\Omega} \left( \sqrt{\frac{n d}{p \log(\overline{B})}} \right) \,,
\end{equation}
for the class of neural networks such that $B(\bw) \leq \overline{B}$. In other words the dependence on the depth all but disappears by simply assuming that the quantity $B(\bw)$ (a natural upper bound on the Lipschitz constant of the network) is polynomially controlled. Interestingly many works have suggested to keep $B(\bw)$ under control, either for regularization purpose (for example \cite{bartlett2017spectrally} relates $B(\bw)$ to the Rademacher complexity of multi-layer neural networks) or to simply control gradient explosion during training, see e.g., \cite{arjovsky2016unitary,cisse2017parseval,mhammedi2017efficient,miyato2018spectral, jiang2018computation,yoshida2017spectral}. Moreover, in addition to being well-motivated in practice, the assumption that $\overline{B}$ is polynomially controlled seems also somewhat unavoidable in theory, since $B(\bw)$ is an {\em upper bound} on the Lipschitz constant $\mathrm{Lip}(f_{\bw})$. Thus a theoretical construction showing that the lower bound in \eqref{eq:withdepth} is tight (at some large depth $D$) would necessarily need to have an exponential gap between $\mathrm{Lip}(f_{\bw})$ and $B(\bw)$. We are not aware of any such example, and it would be interesting to fully elucidate the role of depth in the law of robustness (particularly if it could give recommendation on how to best distribute parameters in a neural network).

\section{Generalization Perspective}

The law of robustness can be phrased in a slightly stronger way, as a generalization bound for classes of Lipschitz functions based on data-dependent Rademacher complexity. In particular, this perspective applies to any Lipschitz loss function, whereas our analysis in the main text was specific to the squared loss. We define the data-dependent Rademacher complexity $\Rad_{n,\mu}(\mathcal F)$ by

\begin{equation}
\Rad_{n,\mu}(\mathcal F)=\frac{1}{n}\E^{\sigma_i,x_i}\left[\sup_{f\in \mathcal F}\left|\sum_{i=1}^{n} \sigma_i f(x_i)\right|  \right]
\end{equation}

where the values $(\sigma_i)_{i\in [n]}$ are i.i.d. symmetric Rademacher variables in $\{-1,1\}$ while the values $(x_i)_{i\in [n]}$ are i.i.d. samples from $\mu$. 

\begin{lemma}\label{lem:gen1}

Suppose $\mu=\sum_{i=1}^k \alpha_i\mu_i$ is a mixture of $c$-isoperimetric distributions. For finite $\mathcal F$ consisting of $L$-Lipschitz $f$ with $|f(x)|\leq 1$ for all $(f,x)\in \mathcal F\times \mathbb R^d$, we have

\begin{equation}
\Rad_{n,\mu}(\mathcal F)\leq O\left(\max\left(\sqrt{\frac{k}{n}},L\sqrt{\frac{c\log(|\mathcal F|)}{nd}}\right)\right) .\end{equation}

\end{lemma}

The proof is identical to that of Theorem~\ref{thm:finite}. Although we do not pursue it in detail, Lemma~\ref{lem:sigma-improve-main} easily extends to a sharpening of this result to general $\sigma_i\in [-1,1]$ when $\mathbb E[\sigma_i^2]$ is small, even if $\sigma_i$ and $x_i$ are not independent. We only require that the $n$ pairs $(\sigma_i,x_i)_{i\in [n]}$ are i.i.d. and that the distribution of $\sigma_i$ given $x_i$ is symmetric. To see that the latter symmetry condition is natural, recall the quantity $\Rad_{n,\mu}$ classically controls generalization due to the symmetrization trick, in which one writes $\sigma_i=y_i-y_i'$ for $y_i'$ a resampled label for $x_i$.
\newline

Note that $\Rad_{n,\mu}(\mathcal F)$ simply measures the ability of functions in $\mathcal F$ to correlate with random noise. Using standard machinery \msedit{(see e.g. \cite[Chapter 3]{mohri2018foundations} for more on these concepts)} we now deduce the following generalization bound:

\begin{corollary}
\label{cor:Rademacher}
For any loss function $\ell(t,y)$ which is bounded and $1$-Lipschitz in its first argument and any $\delta\in [0,1]$, in the setting of Lemma~\ref{lem:gen1} we have with probability at least $1-\delta$ the uniform convergence bound:
\[
\sup_{f\in \mathcal F}\left|\E^{(x,y)\sim \mu}[\ell(f(x),y)]-\frac{1}{n}\sum_{i=1}^n \ell(f(x_i),y_i)\right|\leq O\left(\max\left(\sqrt{\frac{k}{n}},L\sqrt{\frac{c\log(|\mathcal F|)}{nd}},\sqrt{\frac{\log(1/\delta)}{n}}\right)\right) \,.
\]

\end{corollary}

\begin{proof}
Using McDiarmid's concentration inequality it is enough to bound the left hand side in expectation over $(x_i, y_i)$. Using the symmetrization trick \msedit{(see e.g. \cite[Chapter 7]{van2014probability})}, one reduces this task to upper bounding
\[
\E^{x_i,y_i,\sigma_i} \sup_{f \in \mathcal{F}} \frac{1}{n}\sum_{i=1}^n \sigma_i \ell(f(x_i),y_i) \,.
\]
Fixing the pairs $(x_i,y_i)$ and using the contraction lemma (see e.g. \cite[Theorem 26.9]{shalev2014understanding}) the above quantity is upper bounded by $\Rad_{n,\mu}(\mathcal{F})$ which concludes the proof.
\end{proof}

Of course, one can again use an $\epsilon$-net to obtain an analogous result for continuously parametrized function classes. The law of robustness, now for a general loss function, follows as a corollary (the argument is similar to [Proposition 1, \cite{BELM20}]). Let us point out that many papers have studied the Rademacher complexity of function classes such as neural networks (see e.g. \cite{bartlett2017spectrally}, or \cite{yin2019rademacher} in the context of adversarial examples). The new feature of our result is that isoperimetry of the covariates yields improved generalization guarantees.

\section*{Acknowledgement}

M.S. gratefully acknowledges support of NSF grant CCF-2006489, an NSF graduate research fellowship, and a Stanford graduate fellowship. We thank Gene Li, Omar Montasser, Kumar Kshitij Patel, Nati Srebro, and Lijia Zhou for suggesting that the improvement of Lemma~\ref{lem:sigma-improve-main} might be possible for small $\sigma$, and an anonymous referee for pointing out a simpler proof. Thanks also to Franka Exner for pointing out some errors with numerical constants.

\bibliographystyle{alpha}
\bibliography{main}

\newcommand{\etalchar}[1]{$^{#1}$}
\begin{thebibliography}{BHMM19}

\bibitem[ASB16]{arjovsky2016unitary}
Martin Arjovsky, Amar Shah, and Yoshua Bengio.
\newblock Unitary evolution recurrent neural networks.
\newblock In {\em International Conference on Machine Learning}, pages
  1120--1128. PMLR, 2016.

\bibitem[BELM20]{BELM20}
Sebastien Bubeck, Ronen Eldan, Yin~Tat Lee, and Dan Mikulincer.
\newblock Network size and size of the weights in memorization with two-layers
  neural networks.
\newblock In {\em Advances in Neural Information Processing Systems},
  volume~33, pages 4977--4986, 2020.

\bibitem[BFT17]{bartlett2017spectrally}
Peter~L Bartlett, Dylan~J Foster, and Matus~J Telgarsky.
\newblock Spectrally-normalized margin bounds for neural networks.
\newblock In I.~Guyon, U.~V. Luxburg, S.~Bengio, H.~Wallach, R.~Fergus,
  S.~Vishwanathan, and R.~Garnett, editors, {\em Advances in Neural Information
  Processing Systems}, volume~30. Curran Associates, Inc., 2017.

\bibitem[BHMM19]{Belkin15849}
Mikhail Belkin, Daniel Hsu, Siyuan Ma, and Soumik Mandal.
\newblock Reconciling modern machine-learning practice and the classical
  bias{\textendash}variance trade-off.
\newblock {\em Proceedings of the National Academy of Sciences},
  116(32):15849--15854, 2019.

\bibitem[BL97]{bobkov1997poincare}
Sergey Bobkov and Michel Ledoux.
\newblock Poincar{\'e}’s inequalities and talagrand’s concentration
  phenomenon for the exponential distribution.
\newblock {\em Probability Theory and Related Fields}, 107(3):383--400, 1997.

\bibitem[BL00]{bobkov2000brunn}
Sergey~G Bobkov and Michel Ledoux.
\newblock From brunn-minkowski to brascamp-lieb and to logarithmic sobolev
  inequalities.
\newblock {\em Geometric \& Functional Analysis GAFA}, 10(5):1028--1052, 2000.

\bibitem[BLLT20]{Bartlett30063}
Peter~L. Bartlett, Philip~M. Long, G{\'a}bor Lugosi, and Alexander Tsigler.
\newblock Benign overfitting in linear regression.
\newblock {\em Proceedings of the National Academy of Sciences},
  117(48):30063--30070, 2020.

\bibitem[BLN21]{bubeck2020law}
S{\'e}bastien Bubeck, Yuanzhi Li, and Dheeraj~M Nagaraj.
\newblock A law of robustness for two-layers neural networks.
\newblock In {\em Conference on Learning Theory}, pages 804--820. PMLR, 2021.

\bibitem[CBG{\etalchar{+}}17]{cisse2017parseval}
Moustapha Cisse, Piotr Bojanowski, Edouard Grave, Yann Dauphin, and Nicolas
  Usunier.
\newblock {Parseval Networks: Improving Robustness to Adversarial Examples}.
\newblock In {\em International Conference on Machine Learning}, pages
  854--863. PMLR, 2017.

\bibitem[CCNW21]{chen2021dimension}
Hong-Bin Chen, Sinho Chewi, and Jonathan Niles-Weed.
\newblock Dimension-free log-sobolev inequalities for mixture distributions.
\newblock {\em Journal of Functional Analysis}, 281(11):109236, 2021.

\bibitem[FdRL17]{facco2017estimating}
Elena Facco, Maria d’Errico, Alex Rodriguez, and Alessandro Laio.
\newblock Estimating the intrinsic dimension of datasets by a minimal
  neighborhood information.
\newblock {\em Scientific reports}, 7(1):1--8, 2017.

\bibitem[FMN16]{fefferman2016testing}
Charles Fefferman, Sanjoy Mitter, and Hariharan Narayanan.
\newblock Testing the manifold hypothesis.
\newblock {\em Journal of the American Mathematical Society}, 29(4):983--1049,
  2016.

\bibitem[GCL{\etalchar{+}}19]{gao2019convergence}
Ruiqi Gao, Tianle Cai, Haochuan Li, Cho-Jui Hsieh, Liwei Wang, and Jason~D Lee.
\newblock Convergence of adversarial training in overparametrized neural
  networks.
\newblock {\em Advances in Neural Information Processing Systems},
  32:13029--13040, 2019.

\bibitem[GMF{\etalchar{+}}18]{gilmer2018adversarial}
Justin Gilmer, Luke Metz, Fartash Faghri, Samuel~S Schoenholz, Maithra Raghu,
  Martin Wattenberg, and Ian Goodfellow.
\newblock Adversarial spheres.
\newblock {\em arXiv preprint arXiv:1801.02774}, 2018.

\bibitem[GQU{\etalchar{+}}20]{gowal2021uncovering}
Sven Gowal, Chongli Qin, Jonathan Uesato, Timothy Mann, and Pushmeet Kohli.
\newblock Uncovering the limits of adversarial training against norm-bounded
  adversarial examples.
\newblock {\em arXiv preprint arXiv:2010.03593}, 2020.

\bibitem[Gro86]{gromov1986isoperimetric}
Mikhael Gromov.
\newblock {Isoperimetric Inequalities in Riemannian Manifolds}.
\newblock In {\em Asymptotic Theory of Finite Dimensional Spaces}, volume 1200,
  pages 114--129. Springer Berlin, 1986.

\bibitem[HS89]{hastie1989principal}
Trevor Hastie and Werner Stuetzle.
\newblock Principal curves.
\newblock {\em Journal of the American Statistical Association},
  84(406):502--516, 1989.

\bibitem[HVD15]{hinton2015distilling}
Geoffrey Hinton, Oriol Vinyals, and Jeff Dean.
\newblock Distilling the knowledge in a neural network.
\newblock {\em arXiv preprint arXiv:1503.02531}, 2015.

\bibitem[JCC{\etalchar{+}}19]{jiang2018computation}
Haoming Jiang, Zhehui Chen, Minshuo Chen, Feng Liu, Dingding Wang, and Tuo
  Zhao.
\newblock On computation and generalization of gans with spectrum control.
\newblock {\em Proc. of International Conference on Learning Representation
  (ICLR)}, 2019.

\bibitem[KL93]{kambhatla1993fast}
Nanda Kambhatla and Todd~K Leen.
\newblock Fast nonlinear dimension reduction.
\newblock In {\em IEEE International Conference on Neural Networks}, pages
  1213--1218. IEEE, 1993.

\bibitem[Led99]{ledoux1999concentration}
Michel Ledoux.
\newblock Concentration of measure and logarithmic sobolev inequalities.
\newblock In {\em Seminaire de probabilites XXXIII}, pages 120--216. Springer,
  1999.

\bibitem[Led01]{Ledoux}
M.~Ledoux.
\newblock {The concentration of measure phenomenon}.
\newblock In {\em Mathematical Surveys and Monographs}, volume~89. {American
  Mathematical Society, Providence, RI}, 2001.

\bibitem[MHRB17]{mhammedi2017efficient}
Zakaria Mhammedi, Andrew Hellicar, Ashfaqur Rahman, and James Bailey.
\newblock Efficient orthogonal parametrisation of recurrent neural networks
  using householder reflections.
\newblock In {\em International Conference on Machine Learning}, pages
  2401--2409. PMLR, 2017.

\bibitem[MKKY18]{miyato2018spectral}
Takeru Miyato, Toshiki Kataoka, Masanori Koyama, and Yuichi Yoshida.
\newblock Spectral normalization for generative adversarial networks.
\newblock {\em Proc. of International Conference on Learning Representation
  (ICLR)}, 2018.

\bibitem[MM19]{mei2020generalization}
Song Mei and Andrea Montanari.
\newblock The generalization error of random features regression: Precise
  asymptotics and the double descent curve.
\newblock {\em Communications on Pure and Applied Mathematics}, 2019.

\bibitem[MMS{\etalchar{+}}18]{madry2017towards}
Aleksander Madry, Aleksandar Makelov, Ludwig Schmidt, Dimitris Tsipras, and
  Adrian Vladu.
\newblock Towards deep learning models resistant to adversarial attacks.
\newblock {\em Proc. of International Conference on Learning Representation
  (ICLR)}, 2018.

\bibitem[MRT18]{mohri2018foundations}
Mehryar Mohri, Afshin Rostamizadeh, and Ameet Talwalkar.
\newblock {\em Foundations of Machine Learning}.
\newblock MIT press, 2018.

\bibitem[NBA{\etalchar{+}}18]{novak2018sensitivity}
Roman Novak, Yasaman Bahri, Daniel~A. Abolafia, Jeffrey Pennington, and Jascha
  Sohl-Dickstein.
\newblock Sensitivity and generalization in neural networks: an empirical
  study.
\newblock In {\em International Conference on Learning Representations}, 2018.

\bibitem[NKB{\etalchar{+}}20]{Nakkiran2020Deep}
Preetum Nakkiran, Gal Kaplun, Yamini Bansal, Tristan Yang, Boaz Barak, and Ilya
  Sutskever.
\newblock Deep double descent: Where bigger models and more data hurt.
\newblock In {\em International Conference on Learning Representations}, 2020.

\bibitem[NM10]{narayanan2010sample}
Hariharan Narayanan and Sanjoy Mitter.
\newblock Sample complexity of testing the manifold hypothesis.
\newblock In {\em Proceedings of the 23rd International Conference on Neural
  Information Processing Systems-Volume 2}, pages 1786--1794, 2010.

\bibitem[PZA{\etalchar{+}}21]{pope2021intrinsic}
Phil Pope, Chen Zhu, Ahmed Abdelkader, Micah Goldblum, and Tom Goldstein.
\newblock The intrinsic dimension of images and its impact on learning.
\newblock In {\em International Conference on Learning Representations}, 2021.

\bibitem[RS00]{roweis2000nonlinear}
Sam~T Roweis and Lawrence~K Saul.
\newblock Nonlinear dimensionality reduction by locally linear embedding.
\newblock {\em Science}, 290(5500):2323--2326, 2000.

\bibitem[RWK20]{pmlr-v119-rice20a}
Leslie Rice, Eric Wong, and Zico Kolter.
\newblock Overfitting in adversarially robust deep learning.
\newblock In {\em Proceedings of the 37th International Conference on Machine
  Learning}, volume 119 of {\em Proceedings of Machine Learning Research},
  pages 8093--8104. PMLR, 2020.

\bibitem[SSBD14]{shalev2014understanding}
Shai Shalev-Shwartz and Shai Ben-David.
\newblock {\em Understanding machine learning: From theory to algorithms}.
\newblock Cambridge university press, 2014.

\bibitem[TDSL00]{tenenbaum2000global}
Joshua~B Tenenbaum, Vin De~Silva, and John~C Langford.
\newblock A global geometric framework for nonlinear dimensionality reduction.
\newblock {\em Science}, 290(5500):2319--2323, 2000.

\bibitem[Ver18]{vershynin2018high}
Roman Vershynin.
\newblock {\em High-dimensional probability: An introduction with applications
  in data science}, volume~47.
\newblock Cambridge university press, 2018.

\bibitem[vH14]{van2014probability}
Ramon van Handel.
\newblock Probability in high dimension.
\newblock Technical report, Princeton University, 2014.

\bibitem[XY20]{Xie2020Intriguing}
Cihang Xie and Alan Yuille.
\newblock Intriguing properties of adversarial training at scale.
\newblock In {\em International Conference on Learning Representations}, 2020.

\bibitem[YKB19]{yin2019rademacher}
Dong Yin, Ramchandran Kannan, and Peter Bartlett.
\newblock Rademacher complexity for adversarially robust generalization.
\newblock In {\em International conference on machine learning}, pages
  7085--7094. PMLR, 2019.

\bibitem[YM17]{yoshida2017spectral}
Yuichi Yoshida and Takeru Miyato.
\newblock Spectral norm regularization for improving the generalizability of
  deep learning.
\newblock {\em arXiv preprint arXiv:1705.10941}, 2017.

\bibitem[YSJ19]{yun2019small}
Chulhee Yun, Suvrit Sra, and Ali Jadbabaie.
\newblock {Small ReLU networks are powerful memorizers: a tight analysis of
  memorization capacity}.
\newblock In {\em Advances in Neural Information Processing Systems}, pages
  15532--15543, 2019.

\end{thebibliography}


\appendix

\section{Necessity of Polynomially Bounded Weights} \label{sec:polyweights}

In \cite{bubeck2020law} it was conjectured that the law of robustness should hold for the class of {\em all} two-layer neural networks. In this paper we prove that in fact it holds for arbitrary smoothly parametrized function classes, as long as the parameters are of size at most polynomial. In this section we demonstrate that this polynomial size restriction is necessary for bounded depth neural networks. 

First we note that {\em some} restriction on the size of the parameters is certainly necessary in the most general case. Indeed one can build a single-parameter family, where the single real parameter is used to approximately encode all Lipschitz functions from a compact set in $\R^d$ to $[-1,1]$, simply by brute-force enumeration. In particular no tradeoff between number of parameters and attainable Lipschitz constant would exist for this function class.

Showing a counter-example to the law of robustness with unbounded parameters and ``reasonable'' function classes is slightly harder. Here we build a three-layer neural network, with a single fixed nonlinearity $\sigma : \R \to \R$, but the latter is rather complicated and we do not know how to describe it explicitly (it is based on the Kolmogorov-Arnold theorem). It would be interesting to give similar constructions using other function classes such as ReLU networks.

\begin{theorem}\label{thm:KA}

For each $d\in\mathbb Z^+$ there is a continuous function $\sigma:\R\to \R$ and a sequence $(b_{\ell})_{\ell\leq 2^{2^d}}$ such that the following holds. The function $\Phi_a$ defined by

\begin{equation}\label{eq:thmka}\Phi_a(x)=\sum_{\ell=1}^{2^{2^d}}\sigma(a-\ell)\sum_{i=1}^{2d}\sigma\left(b_{\ell}+\sum_{j=1}^d\sigma(x_j+b_{\ell})\right),\quad\quad |a|\leq 2^{2^d}\end{equation} is always $O(d^{3/2})$-Lipschitz, and the parametrization $a\to\Phi_a$ is $1$-Lipschitz. Moreover for $n\leq \frac{2^d}{100}$, given i.i.d. uniform points $x_1,\dots,x_n\in\mathbb S^{d-1}$ and random labels $y_1,\dots,y_n\in\{-1,1\}$, with probability $1-e^{-\Omega(d)}$ there exists $\ell\in [2^{2^d}]$ such that $\Phi_{\ell}(x_i)=y_i$ for at least $\frac{3n}{4}$ of the values $i\in [n]$. 

\end{theorem}

\begin{proof}

For each coordinate $i\in [d]$, define the slab 
\[
  \slab_i=\left\{x\in \mathbb S^{d-1}:|x_i|\leq \frac{1}{100d^{3/2}}\right\}
\]
and set $\slab=\bigcup_{i\in [d]}\slab_i$. Then it is not difficult to see that $\mu(\slab)\leq \frac{1}{10}$. We partition $\mathbb S^{d-1}\backslash\slab$ into its $2^d$ connected components, which are characterized by their sign patterns in $\{-1,1\}^d$; this defines a piece-wise constant function $\gamma:\mathbb S^{d-1}\backslash\slab\to \{-1,1\}^d$. If we sample the points $x_1,\dots,x_n$ sequentially, each point has probability at least $\frac{4}{5}$ to be in a new cell - this implies that with probability $1-e^{-\Omega(n)}$, at least $\frac{3n}{4}$ are in a unique cell. It therefore suffices to give a construction that achieves $\Phi(x_i)=y_i$ for all $x_i\notin \slab$ such that $\gamma(x_i)\neq \gamma(x_j)$ for all $j\in [n]\backslash \{i\}$. We do this now. 

For each of the $2^{2^d}$ functions $g_{\ell}:\{-1,1\}^d \to \{-1,1\}$, we now obtain the partial function $\tilde h_{\ell}=g_{\ell}\circ \gamma:\mathbb S^{d-1}\backslash\slab\to \{-1,1\}.$ By the Kirszbraun extension theorem, $\tilde h_{\ell}$ extends to an $O(d^{3/2})$-Lipschitz function $h_{\ell}:\mathbb S^{d-1}\to [-1,1]$ on the whole sphere. The Kolmogorov-Arnold theorem guarantees the existence of an exact representation

\begin{equation}\label{eq:KA0}\Phi_{\ell}(x)=\sum_{i=1}^{2d}\sigma_{\ell}\left(\sum_{j=1}^d\sigma_{\ell}(x_j)\right)\end{equation} of $h_{\ell}$ by a two-layer neural network for some continuous function $\sigma_{\ell}:\R\to\R$ depending on $\ell$. It suffices to give a single neural network capable of computing all functions $(\Phi_{\ell})_{\ell=1}^{2^{2^d}}$. We extend the definition of $\Phi_{a}$ to any $a\in\R$ via:

\begin{equation}\label{eq:KA}\Phi_a(x)=\sum_{\ell=1}^{2^{2^d}}\sigma(a-\ell)\Phi_{\ell}(x)\end{equation}
where $\sigma:\R\to\R$ satisfies $\sigma(x)=(1-|x|)_+$ for $|x|\leq 2^{2^d}$. This ensures that \eqref{eq:KA} extends \eqref{eq:KA0}. To express $\Phi_a$ using only a single non-linearity, we prescribe further values for $\sigma$. Let \[U=2^{2^d}+d\cdot\max_{x\in [-1,1],\ell\in [2^{2^d}]}|\sigma_{\ell}(x)|\] so that $\left|\sum_{j=1}^d\sigma_{\ell}(x_j)\right|\leq U$ for all $x\in\mathbb S^{d-1}$. Define real numbers $b_{\ell}=10\ell U+2^{2^d}$ for $\ell\in [2^{2^d}]$ and for all $|x|\leq U$ set

\[\sigma(x+b_{\ell})=\sigma_{\ell}(x).\] Due to the separation of the values $b_{\ell}$ such a function $\sigma$ certainly exists. Then we have

\[\Phi_{\ell}(x)=\sum_{i=1}^{2d}\sigma\left(b_{\ell}+\sum_{j=1}^d\sigma(x_j+b_{\ell})\right).\]

Therefore with this choice of non-linearity $\sigma$ and (data-independent) constants $b_{\ell}$, some function $\Phi_{\ell}$ fits at least $\frac{3n}{4}$ of the $n$ data points with high probability, and the functions $\Phi_a$ are parametrized in a $1$-Lipschitz way by a single real number $a\leq 2^{2^d}$.
\end{proof}

\begin{remark}
  The representation~\eqref{eq:thmka} is a three-layer neural network because the $\sigma(a-\ell)$ terms are just matrix entries for the final layer. 
\end{remark}

\begin{remark}
  The construction above can be made more efficient, using only $O(n\cdot 2^{n})$ uniformly random functions $g_{\ell}:\{-1,1\}^d\to \{-1,1\}$ instead of all $2^{2^{\ell}}$. Indeed by the coupon collector problem, this results in all functions from $\{\gamma(x_i):i\in [n]\}\to \{-1,1\}$ being expressable as the restriction of some $g_{\ell}$, with high probability.
\end{remark}

\end{document}